\documentclass[sigconf,authorversion]{acmart}

\usepackage{tikz}
\usepackage{bm}
\usepackage{cleveref}
\usepackage{mathtools}
\usepackage{nicefrac}

\usetikzlibrary{positioning}

\newcommand{\xhdr}[1]{\noindent{\textbf{#1.}}\hspace{1mm}}

\newcommand{\A}{\mathcal{A}}
\newcommand{\R}{\mathbb{R}}
\newcommand{\D}{\mathcal{D}}
\renewcommand{\U}{\mathcal{U}}

\newcommand{\N}{\mathcal{N}}

\renewcommand{\C}{\mathcal{C}}

\DeclareMathOperator*{\argmin}{arg\,min}

\DeclareMathOperator{\E}{E}

\newcommand{\cat}{\textsf{cat}}
\newcommand{\dog}{\textsf{dog}}
\newcommand{\fish}{\textsf{fish}}

\newtheorem{theorem}{Theorem}
\newtheorem{proposition}[theorem]{Proposition}

\newtheorem{observation}{Observation}

\AtBeginDocument{%
  \providecommand\BibTeX{{%
    \normalfont B\kern-0.5em{\scshape i\kern-0.25em b}\kern-0.8em\TeX}}}

\copyrightyear{2021} 
\acmYear{2021} 
\setcopyright{acmlicensed}
\acmConference[KDD '21]{Proceedings of the 27th ACM SIGKDD Conference on Knowledge Discovery and Data Mining}{August 14--18, 2021}{Virtual Event, Singapore}
\setcopyright{none}
\acmBooktitle{Proceedings of the 27th ACM SIGKDD Conference on Knowledge Discovery and Data Mining (KDD '21), August 14--18, 2021, Virtual Event, Singapore}
\acmPrice{15.00}
\acmDOI{10.1145/3447548.3467378}
\acmISBN{978-1-4503-8332-5/21/08}

\acmSubmissionID{rst2783}

\begin{document}

\title{Choice Set Confounding in Discrete Choice}



\author{Kiran Tomlinson}
\affiliation{%
  \institution{Cornell University}
  \city{}
  \country{}}
\email{kt@cs.cornell.edu}

\author{Johan Ugander}
\affiliation{%
  \institution{Stanford University}
    \city{}
  \country{}}
\email{jugander@stanford.edu}

\author{Austin R.\ Benson}
\affiliation{%
  \institution{Cornell University}
  \city{}
  \country{}}
\email{arb@cs.cornell.edu}


\begin{abstract}
Standard methods in preference learning involve estimating the parameters of discrete choice models
from data of selections (choices) made by individuals from a discrete set of alternatives (the \emph{choice set}).
While there are many models for individual preferences,
existing learning methods overlook how choice set assignment affects the data.
Often, the choice set itself is influenced by an individual's preferences;
for instance, a consumer choosing a product from an online retailer is often presented with options from a recommender system that depend on information about the consumer's preferences.
Ignoring these assignment mechanisms can mislead choice models into making biased estimates of preferences,
a phenomenon that we call \emph{choice set confounding}.
We demonstrate the presence of such confounding in widely-used choice datasets.

To address this issue, we adapt methods from causal inference to the discrete choice setting.
We use covariates of the chooser for inverse probability weighting and/or regression controls,
accurately recovering individual preferences in the presence of choice set confounding under certain assumptions.
When such covariates are unavailable or inadequate, 
we develop methods that take advantage of structured 
choice set assignment to improve prediction.
We demonstrate the effectiveness of our methods on real-world choice data, showing, for example, 
that accounting for choice set confounding makes choices observed in 
hotel booking and commute transportation more consistent with rational utility maximization.

\end{abstract}

%
%




\maketitle

\linepenalty=300
\everypar{\looseness=-1}


\section{Introduction}
Individual choices drive the success of businesses and public policy, 
so predicting and understanding them has far-reaching applications in, e.g., 
environmental policy~\cite{brownstone1996transactions}, 
marketing~\cite{allenby1998marketing}, 
Web search~\cite{ieong2012predicting},
and recommender systems~\cite{yang2011collaborative}.
The central task of discrete choice analysis is to learn individual preferences over a set of available items (the \emph{choice set}), 
given observations of people's choices. 
In recent years, machine learning approaches have enabled more accurate choice modeling and prediction~\cite{seshadri2019discovering,rosenfeld2020predicting,bower2020salient,tomlinson2020learning}. 
However, observational choice data analysis has thus far overlooked a crucial fact: the choice set assignment mechanism underlying a dataset can have a significant impact on the generalization of learned choice models, in particular their validity on \emph{counterfactuals}. 

Understanding how new choice sets affect preferences in such counterfactuals is key to many applications, such as
determining which alternative-fuel vehicles to subsidize or which movies to recommend.
In particular, chooser-dependent choice set assignment coupled with heterogeneous preferences can severely mislead choice models,
as they do not model the influence of preferences on choice set assignment.
Recommender systems are one extreme case, where items are selected specifically to appeal to a user.
Such situations also arise in transportation decisions, online shopping, and personalized Web search, resulting in widespread (but often invisible) error in choice models learned from this data.

Drawing on connections with causal inference~\cite{imbens2015causal}, we term the issue of chooser-dependent choice set assignment \emph{choice set confounding}.
Choice set confounding is a major issue for recent machine learning methods whose success is due to capturing deviations from the traditional principles of rational utility maximization that underlie the workhorse multinomial logit model~\cite{mcfadden1974conditional}.
(Unlike older econometric models of ``irrational'' behavior~\cite{tversky1972elimination,wen2001generalized},
these recent methods are practical for modern, large-scale datasets.) 
These deviations are known as \emph{context effects}, and occur whenever the choice set has an influence on a chooser's preferences. Examples include the \emph{asymmetric dominance effect}~\cite{huber1982adding}, where superior options are made to look even better by including inferior alternatives, and the \emph{compromise effect}~\cite{simonson1989choice}, where intermediate options are preferred (e.g., choosing a medium-priced bottle of wine). While context effects are widespread and worth capturing, choice set confounding can result in spurious effects and over-fitting, and
it is unclear if recent machine learning models are learning true effects or simply being misled by chooser-dependent choice set assignment.   

In this paper, we formalize when choice set confounding is an issue and show that it can result in arbitrary systems of choice probabilities, even if choosers are rational utility-maximizers (in contrast, tractable choice models only describe a tiny fraction of possible choice systems). We also provide strong evidence of choice set confounding in two transportation datasets commonly used to demonstrate the presence of context effects and to test new models~\cite{koppelman2006self,seshadri2019discovering,ragain2016pairwise,benson2016relevance}.
Then, to manage choice set confounding, we first adapt two causal inference methods---inverse probability weighting (IPW) and regression controls---to train choice models in the presence of confounding. These methods require chooser covariates satisfying certain assumptions that differ from the traditional causal inference setting. 
For instance, given access to the same covariates used by a recommender system to construct choice sets, we can reweight the dataset to learn a choice model as if choice sets had been user-independent. Alternatively, we can incorporate covariates into the choice model itself, recovering individual preferences as long as those covariates capture preference heterogeneity.

We also show how to manage choice set confounding without such covariates, as many observational datasets have little information about the individuals making choices.
We demonstrate a link between models accounting for context effects and models for choice systems induced by choice set confounding. 
For example, we derive the context-dependent random utility model (CDM)~\cite{seshadri2020learning} from the perspective of choice set confounding,
by treating the choice set as a vector of substitute covariates (e.g., ``someone who is offered item $i$'') in a multinomial logit model.

We develop spectral clustering methods typically used for co-clustering~\cite{dhillon2001co}
that exploit choice set assignment as a signal for chooser preferences,
as a way to improve counterfactual predictions for observed choosers.
To show why and when this can work, we frame the problem of finding sufficient chooser covariates 
as a problem of recovering latent cluster membership in a stochastic block model (SBM) of
the bipartite graph that connects choosers to the items in their choice sets.

In addition to theoretical analysis, we demonstrate the efficacy of our methods on real-world choice data. 
We provide evidence that IPW reduces confounding when modeling hotel booking data, making the choice system more consistent with utility-maximization and making inferred parameters more plausible. 
For example, the confounded data overweights the importance of price, since many users are shown hotels matching their preferences and select the cheapest one. Factors such as star rating would play a more important role in counterfactuals. We also evaluate our clustering approach on online shopping data. By training separate models for different chooser clusters, 
we outperform a mixture model that attempts to discover preference heterogeneity from choices alone, ignoring the signal from choice set assignment.

All of our code, results, and links to our data are available at \\
\centerline{\url{https://github.com/tomlinsonk/choice-set-confounding}.} 

\subsection{Additional related work}
This research is inspired by recent computational advances in learning context-dependent preferences~\cite{seshadri2019discovering,rosenfeld2020predicting,bower2020salient,tomlinson2020learning,pfannschmidt2019learning}.
These methods exhibit strong gains by exploiting context effects but are often evaluated on data with possible choice set confounding. 
Similar confounding issues are well-studied in rating and ranking data within recommender systems~\cite{marlin2007collaborative,schnabel2016recommendations,wang2020causal,wang2019doubly}, but
these approaches do not directly apply to choice data. 
The causal inference ideas that we develop are based on long-standing methods~\cite{imbens2004nonparametric,imbens2015causal};
the challenge we address is how to adapt them for discrete choice data.

The role of choice set assignment does ocassionally appear in the choice literature. For instance, \citeauthor{manski1977structure} used choice set assignment probabilities to derive random utility models~\cite{manski1977structure}. 
More often, traditional choice theory has focused on latent \emph{consideration sets}, 
which are subsets of alternatives that are actually considered by choosers~\cite{ben1995discrete,bierlaire2010analysis} where non-uniform
choice set probabilities play a key role.
 In another setting, \citeauthor{manski1977estimation}~\cite{manski1977estimation} used an approach similar to our inverse probability weighting.
 They were concerned with ``choice-based samples,'' where we first sample an item and then get an observation of a chooser who selected that item (usually, we sample a chooser and then observe their choice)~\cite{manski1977estimation}. 

 The use of regression controls in discrete choice (i.e., including chooser covariates in the utility function) is standard 
 in econometrics~\cite{stratton2008multinomial,bhat2004mixed,train2009discrete}.
 However, in these settings, regression aims to understand how the attributes of an individual affect decision-making, which can unknowingly and accidentally help with confounding.
 This may explain why choice set confounding has not been widely recognized (additionally, in an interview, Manksi discusses that choice set generation has been under-explored~\cite{tamer_2019}).
 We formalize when and how regression adjusts for choice set confounding.


\section{Discrete choice background}
We start with some notation and the basics of discrete choice models.
Let $\U$ denote a universe of $n$ items and $\A$ a population of individuals.
In a discrete choice setting, a \emph{chooser} $a\in \A$ is presented a nonempty \emph{choice set} $C \subseteq \U$ and they choose one item $i \in C$. 
Specifically, $a$ is sampled with probability $\Pr(a)$, then $C$ is presented to $a$ with probability $\Pr(C \mid a)$, and finally $a$ selects $i$ with probability $\Pr(i \mid a, C)$. 
Most discrete choice analysis focuses only on $\Pr(i \mid a, C)$ or $\Pr(i \mid C)$, but we consider this entire process.
A discrete choice dataset $\D$ is a collection of tuples $(C, i)$ generated by this process. We use $\C_\D$ to denote the set of unique choice sets in $\D$.

Discrete choice models posit a parametric form for choice probabilities, with parameters learned from data. The \emph{universal logit}~\cite{mcfadden1977application} can express any system of choice probabilities (called a \emph{choice system}). Under a universal logit, each chooser $a$ has a scalar utility $u_i(C, a)$ for item $i$ in choice set $C$. Choice probabilities are then a softmax of these utilities: 
$\Pr(i \mid a, C) = \nicefrac{\exp(u_i(C, a))}{\sum_{j \in C} \exp(u_j(C, a))}.$
This arises from a notion of rational utility-maximization~\cite{train2009discrete}. 
Specifically, these are the choice probabilities if $a$ observes random utilities $u_i(C, a) + \epsilon$ (where the $\epsilon$ are i.i.d.\ Gumbel-distributed for each item and choice) and selects the item with maximum observed utility.
The above model has too many degrees of freedom to be practical (e.g., it has entirely separate parameters for every chooser $a$), 
and typically one assumes utilities are fixed across sets and individuals. 
This is the \emph{logit} model~\cite{mcfadden1974conditional}, where $u_i(C, a) = u_i, \forall C, \forall a$. 

Other discrete choice models come from different assumptions on $u_i(C, a)$, trading off descriptive power for ease of inference and interpretation. 
For example, we may have access to a vector of covariates $\bm{x_a} \in \R^{d_x}$ for person $a$. Similarly, an item $i$ may be described by a vector of features $\bm{y_i} \in \R^{d_y}$. We can write $u_i(C, a)$ as a function of $\bm{x_a}, \bm{y_i},$ or both, yielding several choice models (\Cref{tab:models})---the multinomial logit (MNL), conditional logit (CL), and conditional multinomial logit (CML).\footnote{``MNL'' sometimes refers to logit and conditional logit. Here, we follow the convention~\cite{hoffman1988multinomial} that ``multinomial'' means chooser covariates are used and ``conditional'' means item features are used. Additionally, for CML, we assume $\bm{\gamma_i} = B^T \bm{x_a}$, which reduces the number of parameters from $d_y+nd_x$ to $d_y(d_x+1)$, allowing us to use the model when the number of items is prohibitively large.}
All of these models obey a common assumption, the \emph{independence of irrelevant alternatives} (IIA)~\cite{train2009discrete}. IIA states that relative choice probabilities are conserved across choice sets:
$\nicefrac{\Pr(i \mid a, C)}{\Pr(j \mid a, C)} = \nicefrac{\Pr(i \mid a, C')}{\Pr(j \mid a, C')}$.
To be precise, this is individual-level rather than group-level IIA. Among the models in \Cref{tab:models}, the latter is only obeyed by the logit and conditional logit. 
In general, models obey individual-level IIA if utility is independent of $C$, i.e., $u_i(C, a) = u_i(a)$ and obey group-level IIA 
if $u_i(C, a)$ is independent of both $C$ and $a$.

\begin{table}[t]
	\caption{Discrete choice models. The item and chooser feature vectors $\bm{y_i}$ and $\bm{x_a}$ are part of the dataset, while $u_i \in \R, \bm{\theta}\in \R^{d_y}, \bm{\gamma_i}\in \R^{d_x},$ and $B \in \R^{d_y \times d_x}$ are learned parameters.}\label{tab:models}
	\begin{tabular}{lll}
	\toprule
	\bfseries{Model} & $u_i(C, a)$ & \bfseries{\#  Parameters}\\
	\midrule
	logit & $u_i$ & $n$\\
	multinomial logit (MNL) & $u_i+\bm{x_a}^T \bm{\gamma_i}$ & $n(d_x+1)$\\
	conditional logit (CL) & $\bm{y_i}^T \bm{\theta} $ & $d_y$\\
	cond.\ mult.\ logit (CML) & $\bm{y_i}^T (\bm{\theta} + B \bm{x_a})$ & $d_y(d_x+1)$\\
	\bottomrule
	\end{tabular}
\end{table}

While the IIA assumption is convenient, 
it is commonly violated through context effects~\cite{huber1982adding,simonson1992choice,benson2016relevance}. 
Due to the ubiquity of context effects, models incorporating information from the choice set have become increasingly popular and have shown considerable success~\cite{seshadri2019discovering,rosenfeld2020predicting,bower2020salient,tomlinson2020learning}. Other models allow IIA violations without explicitly modeling effects of the choice set~\cite{ragain2016pairwise,mcfadden2000mixed,benson2016relevance}.

We briefly introduce two of these context effect models, the \emph{context-dependent random utility model} (CDM)~\cite{seshadri2019discovering} and the \emph{linear context logit} (LCL)~\cite{tomlinson2020learning}, that we use extensively.
In the CDM, each item in the choice set exerts a pull on the utility of every other item: $u_i(C, a) = \sum_{j \in C\setminus i} p_{ij}$. The CDM can be derived as a second-order approximation to universal logit (where plain logit is the first-order approximation)~\cite{seshadri2019discovering}. The LCL instead operates in settings with item features, adjusting the conditional logit parameter $\bm{\theta}$ according to a linear transformation of the choice set's mean feature vector: $u_i(C, a) = \bm{y_i}^T (\bm{\theta} + A\bm{y_C}) $, where $\bm{y_C} = 1/|C| \sum_{j \in C}\bm{y_j}$. To incorporate chooser covariates, we define multinomial versions of these models (\Cref{tab:context_models}). 
For this paper, LCL and CDM should be thought of as the simplest context effect models with and without item features.

\begin{table}[t]
	\caption{Context effect models. $p_{ij}\in \R, \bm{\gamma_i}\in \R^{d_x}, \bm{\theta} \in \R^{d_y}, A\in \R^{d_y\times d_y}, B\in \R^{d_y\times d_x}$ are learned parameters.}\label{tab:context_models}
	\begin{tabular}{lll}
	\toprule
	\bfseries{Model} & $u_i(C, a)$ & \bfseries{\# Parameters}\\
	\midrule
	CDM~\cite{seshadri2019discovering} & $\sum_{j \in C\setminus i} p_{ij}$ & $n(n-1)$\\
	mult.\ CDM (MCDM) & $\sum_{j \in C\setminus i} p_{ij} + \bm{x_a}^T \bm{\gamma_i}$ & $n(n + d_x)$\\
	LCL~\cite{tomlinson2020learning} & $\bm{y_i}^T (\bm{\theta} + A\bm{y_C})$ & $d_y(d_y+1)$\\
	mult.\ LCL (MLCL) & $\bm{y_i}^T (\bm{\theta} + A\bm{y_C} + B\bm{x_a})$ & $d_y(d_y+d_x+1)$\\
	\bottomrule
	\end{tabular}
\end{table}

In contrast, \emph{mixed logit}~\cite{mcfadden2000mixed} accounts for group-level rather than individual-level IIA violations and has a different structure than any of the models introduced thus far. A (discrete) mixed logit is a mixture of $K$ logits with mixing proportions $\pi_1, \dots, \pi_K$ such that $\sum_{k=1}^K \pi_k = 1$. With $u_i(a_k)$ denoting the utility of the $k$th component for item $i$, a mixed logit has choice probabilities
\begin{equation}\label{eq:mixed_logit}
\Pr(i \mid C) = \sum_{k =1}^K \pi_k \frac{\exp(u_i(a_k))}{\sum_{j \in C} \exp(u_j(a_k))}.
\end{equation}
This can result in a choice system violating IIA but not because any individual chooser experiences context effects. 
Rather, the aggregation of several choosers each obeying IIA can result in IIA violations.


\section{Choice set confounding}\label{sec:wild}

The traditional approach to choice modeling is to learn a single model for $\Pr(i \mid C)$ (such as a logit) and assume it represents overall choice behavior, namely, that the model accurately reflects average choice probabilities $\E_a[\Pr(i \mid a, C)]$. However, $\Pr(i \mid C)$ need not represent average choice behavior at all, as this is only guaranteed under restrictive independence assumptions.
\begin{observation}\label{obs:unbiased}
	If, for all $a\in \A, C \in \C_\D, i \in C$, at least one of
	\begin{enumerate}
		\item $\Pr(C) = \Pr(C \mid a)$ (chooser-independent choice sets) or
		\item $\Pr(i \mid a, C) = \Pr(i \mid C)$ (chooser-independent preferences)
	\end{enumerate}
	holds, then $\Pr(i \mid C) = \E_a[\Pr(i \mid a, C)$.
	If both conditions are violated, then this equality can fail.
\end{observation}
\Cref{app:proofs} has a proof for this fact and other theoretical statements presented later. 
When we have both chooser-dependent sets and preferences, observed choice probabilities $\Pr(i\mid C)$ can differ significantly from true aggregate choice probabilities $\E_a[\Pr(i \mid a, C)]$. We call this phenomenon \emph{choice set confounding}, and
provide the following toy example as an illustration.
\begin{example}\label{ex:pets}
Let $\U = \{\cat, \dog, \fish\}$. Choosers are either cat people or dog people choosing a pet, with choice probabilities
\begin{center}
  \begin{tabular}{rcc}
  & $\{\cat, \dog\}$ & \{\cat, \dog, \fish\}\\
  cat person & \nicefrac{3}{4}, \nicefrac{1}{4} & \nicefrac{3}{4}, \nicefrac{1}{4}, 0 \\
  dog person & \nicefrac{1}{4}, \nicefrac{3}{4} & \nicefrac{1}{4}, \nicefrac{3}{4}, 0
  \end{tabular}
\end{center}
Note that the preferences of cat and dog people do not change when fish are included in the choice set.
Choice sets are assigned non-independently: cat people see $\{\cat, \dog\}$ w.p.\ \nicefrac{3}{4} and $\{\cat, \dog, \fish\}$ w.p.\ \nicefrac{1}{4} (vice-versa for dog people). Let the population consist of \nicefrac{1}{4} cat people and \nicefrac{3}{4} dog people. 
If we only observe samples $(C, i)$ without knowing who is a cat person and who is a dog person,
\begin{align*}
 \Pr(\dog \mid \{\cat, \dog\}) &= \nicefrac{1}{2}\cdot \nicefrac{1}{4} + \nicefrac{1}{2} \cdot \nicefrac{3}{4} = \nicefrac{1}{2}\\
  \Pr(\dog \mid \{\cat, \dog, \fish\}) &=\nicefrac{1}{10}\cdot\nicefrac{1}{4} + \nicefrac{9}{10}\cdot \nicefrac{3}{4} = \nicefrac{7}{10}.
\end{align*}
However, 
\begin{align*}
  \E_a[\Pr(\dog \mid a, \{\cat, \dog\})] &= \nicefrac{1}{4}\cdot\nicefrac{1}{4} + \nicefrac{3}{4}\cdot \nicefrac{3}{4} = \nicefrac{5}{8}\\
  \E_a[\Pr(\dog \mid a, \{\cat, \dog, \fish\}) &= \nicefrac{1}{4}\cdot\nicefrac{1}{4} +\nicefrac{3}{4} \cdot \nicefrac{3}{4} = \nicefrac{5}{8}.
\end{align*}
\end{example}
This mismatch is especially problematic for models that use choice-set dependent utilities $u_i(C)$, such as those designed to account for context effects. From the above data, we might conclude that the presence of a \fish{} causes a \dog{} to become a more appealing option. This spurious context effect would be seized upon by context-based models and even result in improved predictive performance on test data drawn form the same distribution. However, these models would make biased predictions on counterfactual examples where sets are chosen from a different distribution.

In reality, no one's choice would be affected by adding \fish{} to their choice set---it's a red herring. This is a \emph{causal inference} problem. 
We want to know the cause of a choice, but we are being misled as to whether the change in preferences between the $\{\cat, \dog\}$ and $\{\cat, \dog, \fish\}$ choice sets is due to the presence of \fish{} or to a hidden confounder: the underlying preferences of cat and dog people, coupled with chooser-dependent choice set assignment.

Extending this idea, the equality in \Cref{obs:unbiased} can fail dramatically. 
If the population consists of individuals each of whom obeys IIA (i.e., chooses according to a logit), then $\E_a[\Pr(i \mid a, C)]$ is exactly the mixed logit choice probability. On the other hand, $\Pr(i \mid C)$ can express an arbitrary choice system with choice set confounding.

\begin{theorem}\label{thm:confounding_strength}
  Mixed logit with chooser-dependent choice sets is powerful enough to express any system of choice probabilities.
\end{theorem}

Arbitrary choice systems are much more powerful than mixed logit (even ones with continuous mixtures). For example, it is impossible for  mixed logit to violate \emph{regularity}, the condition that $\Pr(i \mid C) \ge \Pr(i \mid C \cup \{j\})$ for all $C \subseteq \U, i\in C, j \in \U$, as
choice probabilities for $i$ can only go down in each mixture component when we include $j$. 
On the other hand, even \Cref{ex:pets} has a regularity violation (picking a dog is more likely when a fish is available),
despite there being only two types of choosers, both adhering to IIA.


We have shown that choice set confounding is an issue in theory, and we now demonstrate it to be a problem in practice. 
We present evidence of choice set confounding in two transportation choice datasets, \textsc{sf-work} and \textsc{sf-shop}~\cite{koppelman2006self}. 
These datasets consist of San Francisco (SF) resident surveys for preferred transportation mode to work or shopping, where the choice set is the set of modes available to a respondent. The SF datasets are common testbeds for choice models violating IIA~\cite{koppelman2006self,seshadri2019discovering,ragain2016pairwise,benson2016relevance} and in choice applications~\cite{tomlinson2020optimizing,agarwal2018accelerated}. 

The SF data have regularity violations (\Cref{tab:sf-reg} in \Cref{sec:app_experiment}), ruling out the possibility that the IIA violations in these datasets are just due to mixtures of choosers obeying IIA. 
Thus, these datasets either have (1) true context effects or (2) choice set confounding. So far, the literature has focused on (1), but we argue that (2) is more likely. We compare the likelihoods of logit, MNL, CDM, and MCDM (recall \Cref{tab:models,tab:context_models}) on these datasets through likelihood-ratio tests (\Cref{tab:lrts}). MNL and MCDM both account for chooser-dependent preferences through covariates, while CDM and MCDM both account for context effects. With true context effects, we would expect CDM to be significantly more likely than logit and MCDM to be significantly more likely than MNL. However, this is not the case. While CDM is significantly more likely than logit, MCDM is not significantly more likely than MNL in both SF datasets. Thus, context effects only appear significant before controlling for preference heterogeneity through covariates. This is exactly what we would expect if the IIA violations in these datasets are due to choice set confounding rather than context effects. In contrast, we see significant context effects in the \textsc{expedia} hotel-booking dataset~\cite{kaggle2013expedia} even after controlling for covariates (this dataset uses item features, hence the different models in \Cref{tab:lrts}), so context effects are likely.
This dataset consists of search results (choice sets) and hotel bookings (choices), and we explore it further in \Cref{sec:expedia}.

The choice set confounding leads to a key question: how were choice sets constructed in \textsc{sf-work} and \textsc{sf-shop}? 
According to \citeauthor{koppelman2006self}~\cite{koppelman2006self}, choice sets were imputed based on chooser covariates from the survey. 
For instance, walking was included as an option if a respondent's distance to the destination was $< 4$ miles 
and driving was included if they had a driver's license and at least one car in their household~\cite{koppelman2006self}. This choice set assignment is highly chooser-dependent, resulting in strong choice set confounding.

\begin{table}
    \caption{Likelihood gains in \textsc{sf-work}, \textsc{sf-shop}, and \textsc{expedia} from covariates and context with likelihood ratio test (LRT) $p$-values. $\Delta \ell$ denotes improvement in log-likelihood.}\label{tab:lrts}
    \begin{tabular}{lllrr}
    \toprule
    \bfseries{Comparison} & \bfseries{Testing} &\bfseries{Controlling} &  \bfseries{$\Delta \ell$} & \bfseries{LRT $p$}\\
    \midrule
\textsc{sf-work} & & & &\\
Logit to MNL & covariates & --- & 883 & $< 10^{-10}$\\
Logit to CDM & context & --- & 85 & $< 10^{-10}$\\
CDM to MCDM & covariates & context & 819 & $< 10^{-10}$\\
MNL to MCDM & context & covariates & 20 & 0.08\\
\midrule 
\textsc{sf-shop} & & & &\\
Logit to MNL & covariates & --- & 343 & $< 10^{-10}$\\
Logit to CDM & context & --- & 96 & $< 10^{-10}$\\
CDM to MCDM & covariates & context & 276 & $< 10^{-10}$\\
MNL to MCDM & context & covariates & 29 & 0.36\\
\midrule
\textsc{expedia} & & & &\\
CL to CML & covariates & --- & 1218 & $< 10^{-10}$\\
CL to LCL & context & --- & 2345 & $< 10^{-10}$\\
LCL to MLCL & covariates & context & 1167 & $< 10^{-10}$\\
CML to MLCL & context & covariates & 2294 & $< 10^{-10}$\\
    \bottomrule
    \end{tabular}
\end{table}

\Cref{ex:pets} and the SF datasets highlight how confounding can lead to spurious context effects and 
incorrect average choice probabilities. 
Next, in \Cref{sec:causal_inference}, we adapt methods from causal inference so that chooser covariates can correct choice probability estimates. 
And in \Cref{sec:other_methods}, we address what can be done without covariates if we want to (1) make predictions under chooser-dependent choice set assignment mechanisms or (2) make counterfactual predictions for previously observed choosers. 


\section{Causal inference methods}\label{sec:causal_inference}

In traditional causal inference~\cite{rubin1974estimating,imbens2004nonparametric,imbens2015causal}, we wish to estimate the causal effect of an intervention (e.g., a medical treatment) from observational data. However, we cannot simply compare the outcomes of the treated and untreated cohorts if treatment was not randomly assigned---confounders might affect both whether someone was treated and their outcome. There are many methods to debias treatment effect estimation, including matching~\cite{rubin1974estimating,rosenbaum1983central}, inverse probability weighting (IPW)~\cite{hirano2003efficient}, and regression~\cite{rubin1977assignment}. 
One can also combine methods, such as IPW and regression, which is the basis for \emph{doubly robust} estimators~\cite{bang2005doubly}. 

Here, we adapt causal inference methods to estimate unbiased discrete choice models from data with choice set confounding. First, we adapt
IPW to learn unbiased models that do not use chooser covariates in the utility function. After, we show an equivalence between incorporating chooser covariates in the utility function and regression for causal inference. Finally, we combine these methods for doubly robust choice model estimation. 
For discrete choice, these methods require new assumptions and have different guarantees. 
We first provide a brief introduction to causal inference terminology in the binary treatment setting, such as an observational medical study (in contrast, we will think of choice sets as treatments).   

In potential outcomes notation~\cite{rubin2005causal}, each person $i$ has covariates $X_i$ and is either treated ($T_i=1$) or untreated ($T_i=0$). At some point after treatment, we measure the \emph{outcome} $Y_i(T_i)$. A typical goal of the causal inference methods above is to estimate the \emph{average treatment effect} $\E_i[Y_i(1) - Y_i(0)]$.  All of these methods rely on untestable assumptions; in particular, they rely on \emph{strong ignorability}~\cite{rosenbaum1983central,imbens2004nonparametric} (also called \emph{unconfoundedness} or \emph{no unmeasured confounders}), which requires that the treatment is independent from the outcome, conditioned on observed covariates: $\Pr(T_i\mid X_i, Y_i) = \Pr(T_i\mid X_i), \forall i$.

\subsection{Inverse probability weighting}
IPW estimation commonly requires estimating \emph{propensity scores} describing the probability of each treatment assignment given individual covariates. The true probabilities $\Pr(T_i \mid X_i)$ are unknown, so estimated ``propensities'' $\widehat \Pr(T_i \mid X_i)$ are learned from observed data, typically via logistic regression~\cite{austin2011introduction}. Propensities can then be used to estimate average treatment effects or, as in our case, to re-weight a model's training data~\cite{freedman2008weighting}. By weighting each sample by the inverse of its propensity, we effectively construct a \emph{pseudo-population} where treatment is assigned independently from covariates. In addition to ignorability, IPW requires \emph{positivity}, the assumption that all propensities satisfy $0 < \Pr(T_i\mid X_i) < 1$.

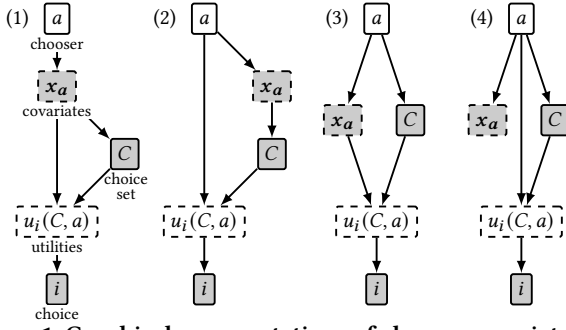
\begin{figure}[t]
\centering
\scalebox{0.9}{
\begin{tikzpicture}[%
  ->,
  thick,
  >=latex,
  node distance=7mm,
  boxed/.style={%
    rectangle, rounded corners=1pt,
    text centered,
        text height=1.5ex,
    text depth=.25ex,
    text centered,
    minimum height=1em,
    draw=black
  },
  baseline=(current bounding box.north)
]
    \node[boxed] (a) at (0,4) {$a$};
    \node[boxed, dashed, fill=black!20] (xa) at (0,3) {$\bm{x_a}$};
    \node[boxed, fill=black!20] (C) at (1,2) {$C$};

    \node[boxed, dashed] (u) at (0,1) {$u_{i}(C, a)$};
    \node[boxed, fill=black!20] (i) at (0,0) {$i$};

    \node[left = 1mm of a] (label) {(1)};

    \path [draw,->, shorten <= 2pt] (xa) edge (C);
    \path [draw,->, shorten <= 2pt] (a) edge (xa);
    \path [draw,->, shorten <= 2pt] (xa) edge (u);
    \path [draw,->] (C) edge (u); 
    \path [draw,->, shorten <= 2pt] (u) edge (i);

    \node[fill=white, inner sep=1pt] (alabel) at (0, 3.64) {\footnotesize chooser};
    \node[fill=white, inner sep=1pt] (calabel) at (0, 2.64) {\footnotesize covariates};
    \node[inner sep=0.5pt,align=center, anchor=north] (Clabel) at (1.02, 1.76) {\footnotesize choice\\[-0.5em]\footnotesize set};
    \node[fill=white, inner sep=1pt] (ulabel) at (0, 0.64) {\footnotesize utilities};
    \node[fill=white, inner sep=1pt] (ilabel) at (0, -0.36) {\footnotesize choice};

\end{tikzpicture}
\hspace*{-2mm}
\begin{tikzpicture}[%
  ->,
  thick,
  >=latex,
  node distance=7mm,
  boxed/.style={%
    rectangle, rounded corners=1pt,
    text centered,
        text height=1.5ex,
    text depth=.25ex,
    text centered,
    minimum height=1em,
    draw=black
  },
  baseline=(current bounding box.north)
]
    \node[boxed] (a) at (0,4) {$a$};
    \node[boxed, dashed, fill=black!20] (xa) at (1,3) {$\bm{x_a}$};
    \node[boxed, fill=black!20] (C) at (1,2) {$C$};

    \node[boxed, dashed] (u) at (0,1) {$u_{i}(C, a)$};
    \node[boxed, fill=black!20] (i) at (0,0) {$i$};

      \node[left = 1mm of a] (label) {(2)};

    \path [draw,->] (a) edge (xa);
    \path [draw,->] (xa) edge (C);  
    \path [draw,->] (a) edge (u);
    \path [draw,->] (C) edge (u); 
    \path [draw,->] (u) edge (i);
\end{tikzpicture}
\quad
\begin{tikzpicture}[%
  ->,
  thick,
  >=latex,
  node distance=7mm,
  boxed/.style={%
    rectangle, rounded corners=1pt,
    text centered,
        text height=1.5ex,
    text depth=.25ex,
    text centered,
    minimum height=1em,
    draw=black
  },
  baseline=(current bounding box.north)
]
    \node[boxed] (a) at (0.5,4) {$a$};
    \node[boxed, dashed, fill=black!20] (xa) at (0,2.5) {$\bm{x_a}$};
    \node[boxed, fill=black!20] (C) at (1,2.5) {$C$};

    \node[boxed, dashed] (u) at (0.5,1) {$u_{i}(C, a)$};
    \node[boxed, fill=black!20] (i) at (0.5,0) {$i$};

    \node[left = 1mm of a] (label) {(3)};
    
    \path [draw,->] (a) edge (C);
    \path [draw,->] (a) edge (xa);
    \path [draw,->] (xa) edge (u);
    \path [draw,->] (C) edge (u); 
    \path [draw,->] (u) edge (i);
\end{tikzpicture}
\quad\;
\begin{tikzpicture}[%
  ->,
  thick,
  >=latex,
  node distance=7mm,
  boxed/.style={%
    rectangle, rounded corners=1pt,
    text centered,
        text height=1.5ex,
    text depth=.25ex,
    text centered,
    minimum height=1em,
    draw=black
  },
  baseline=(current bounding box.north)
]
    \node[boxed] (a) at (0.5,4) {$a$};
    \node[boxed, dashed, fill=black!20] (xa) at (0,2.5) {$\bm{x_a}$};
    \node[boxed, fill=black!20] (C) at (1,2.5) {$C$};

    \node[boxed, dashed] (u) at (0.5,1) {$u_{i}(C, a)$};
    \node[boxed, fill=black!20] (i) at (0.5,0) {$i$};

      \node[left = 1mm of a] (label) {(4)};

    \path [draw,->] (a) edge (C);
    \path [draw,->] (a) edge (xa);
    \path [draw,->] (a) edge (u);
    \path [draw,->] (C) edge (u); 
    \path [draw,->] (u) edge (i);
\end{tikzpicture}
}
\caption{Graphical representations of chooser covariate assumptions: (1) ignorability; (2) choice set ignorability; (3) preference ignorability; (4) no ignorability. Shaded nodes are observed, dashed nodes are deterministic.}\label{fig:graphical_models}
\end{figure}

In the discrete choice setting, we think of choice sets as treatments. By \Cref{obs:unbiased}, we need chooser-independent choice sets in order to learn an unbiased choice model.
Our idea of IPW for discrete choice is to create a pseudo-dataset in which this is true and to learn a choice model over that pseudo-dataset. 
To do this, we model choice set assignment probabilities $\Pr(C \mid a)$. We can then replace each sample $(i, a, C)$ with $1 / [|\C_\D| \Pr(C \mid a)]$ copies, creating a pseudo-dataset $\tilde \D$ with uniformly random choice sets (note that we allow ``fractional samples,'' since we don't explicity construct $\tilde \D$). However, we cannot hope to learn $\Pr(C \mid a)$ in datasets with only a single observation per chooser (which is very often the case). We instead need to rely on observed covariates $\bm{x_a}$. We thus learn $\Pr(C \mid \bm{x_a})$ and use these propensities to construct $\tilde \D$. For the analysis, we assume we know the true propensities, but a correctly specified choice set assignment model learned from data is sufficient. 
 
To learn a choice model from $\tilde \D$, we can simply add weights to the model's log-likelihood function, resulting in
\begin{equation}
    \ell(\theta; \tilde \D) = \sum_{(i, C, a) \in \D} \frac{\log \Pr_\theta(i \mid C)}{|\C_\D|\Pr(C \mid \bm{x_a})}.\label{eq:choice_ipw}
\end{equation} 

In order for $\Pr(C \mid \bm{x_a})$ to be an effective stand-in for $\Pr(C \mid a)$, we need the following assumption (see \Cref{fig:graphical_models}).

\begin{definition}
  \emph{Choice set ignorability} is satisfied if choice sets are independent of choosers, conditioned on chooser covariates: $\Pr(C \mid a, \bm{x_a}) = \Pr(C \mid \bm{x_a})$.
\end{definition}

Just as in standard IPW, we also need positivity (of choice set propensities). Under these assumptions, IPW guarantees that empirical choice probabilities in the pseudo-dataset $\tilde \D$ reflect aggregate choice probabilities in the true population. To formalize this, we introduce $\D^*$, an idealized dataset with uniformly random choice set assignment for every chooser (of the same size as $\D$). $\D^*$ consists of $|\D|$ independent samples $(a, C, i)$ each occuring with probability $\Pr(a)\frac{1}{|\C_\D|}\Pr(i \mid a, C)$. We now show that the IPW-weighted log-likelihood (\cref{eq:choice_ipw}) is, in expectation, the same as the log-likelihood function over $\D^*$. Since $\D^*$ has chooser-independent choice sets, we can train a model for $\Pr(i \mid C)$ using \cref{eq:choice_ipw} and expect it to capture unbiased aggregate choice probabilities (by \Cref{obs:unbiased}).

\begin{theorem}\label{thm:ipw}
  If, for all $a \in \A, C \in \C_\D$, 
  \begin{enumerate}
    \item $0 < \Pr(C \mid \bm{x_a}) < 1$ (positivity), and
    \item $\Pr(C \mid a, \bm{x_a}) = \Pr(C \mid \bm{x_a})$ (choice set ignorability),
  \end{enumerate}
  then  $E_\D[\ell(\theta; \tilde \D)] = E_{\D^*}[\ell(\theta; \D^*)]$.
\end{theorem}
Choice set ignorability is crucial to the success of IPW, so we should assess when this assumption is reasonable. If choice sets are generated by an exogenous process (such as a recommender system, as in the \textsc{expedia} dataset), then as long as we have access to the same covariates as that process, choice set ignorability holds, although learning the propensities may still be a challenge. 
However, in other datasets, choice sets are formed through self-directed browsing (e.g., clicking around an online shop, as in the \textsc{yoochoose} dataset we examine later). 
In those cases, basic covariates (age, gender, etc.) are unlikely to fully capture choice set generation, since sets result from the complexities of human behavior rather than the simpler algorithmic behavior of a recommender system.
As in traditional causal inference, the validity of choice set ignorability must be determined by the practitioner applying the method. 

\subsection{Regression}
An alternative to using chooser covariates to learn choice set propensities is to incorporate covariates directly into the utility formulation, as in the multinomial or conditional multinomial logit models. If chooser covariates fully capture their preferences and the choice model is correctly specified, then the model that we learn is consistent. We formalize the first condition as follows (see \Cref{fig:graphical_models}).
\begin{definition}
  \emph{Preference ignorability} is satisfied if choice probabilities are independent of choosers, conditioned on chooser covariates: $\Pr(i \mid a, \bm{x_a}, C) = \Pr(i \mid \bm{x_a}, C) $.
\end{definition}
Given correct specification and preference ignorability, the choice model will be consistent in terms of aggregate choice probabilities and result in accurate individual choice probability estimates.
\begin{theorem}\label{thm:regression}
  If $\Pr(i \mid a, \bm{x_a}, C) = \Pr(i \mid \bm{x_a}, C)$ for all $a\in \A, C\in \C_\D, i \in C$ (preference ignorability), then the MLE of a correctly specified (and well-behaved, in the standard MLE sense~\cite[Theorem 9.13]{wasserman2013all}) choice model that incorporates chooser covariates $\bm{x_a}$ is consistent:
$\lim_{|\D| \rightarrow \infty}\widehat \Pr(i \mid \bm{x_a}, C) = \Pr(i \mid a, C)$.
\end{theorem}

While the guarantee of regression is stronger than IPW, preference ignorability is more challenging to satisfy in practice. Instead of needing all covariates used to generate choice sets, we need covariates to fully describe choice behavior.

\subsection{Doubly robust estimation}
A constraint of both IPW and regression is correct model specification, either of the choice set propensity model or of the choice model. In traditional causal inference, one can combine both methods to provide guarantees if either model is correctly specified, producing \emph{doubly robust} estimators~\cite{bang2005doubly,funk2011doubly}. 
In the same way, we can combine IPW and regression for choice models and achieve their respective guarantees if their respective conditions are satisfied. In other words, the two methods do not interfere with each other.
However, this increases the variance of estimates, so it may be advisable to only use one method if we are confident in one of the assumptions.

\subsection{Empirical analysis of IPW and regression}\label{sec:ipw-emprical}

We begin by evaluating regression and IPW adjustments in synthetic data, and then apply our methods to the \textsc{expedia} dataset (training details in \Cref{sec:app_experiment}).

\xhdr{Counterfactual evaluation in synthetic data}
We generate synthetic data with heterogeneous preferences, CDM-style context effects, and choice set confounding. 
Specifically, we use 20 items with embeddings $\bm{y_i}\in \R^2$ sampled uniformly from the unit circle. We also generate embeddings $\bm{x_a}$ in the same way for each chooser $a$. Each chooser $a$ picks items according to an MCDM, where the utility for $i$ is a sum of $\bm{x_a}^T \bm{y_i}$ plus a CDM term shared by all choosers, with each ``push/pull'' term $p_{ij} \sim \text{Uniform}(-1, 1)$. 
To generate a choice set for $a$, we sample a uniformly random set with probability $0.25$ (to satisfy positivity) and otherwise include each item with probability 
$\nicefrac{1}{(1+e^{-c \bm{x_a}^T\bm{y_i}})}$,
 where $c$ is the \emph{confounding strength} (we condition on having at least two items in the choice set). 
 Higher confounding strength results in sets containing items more preferred by $a$. 
 Each trial consists of 10000 samples. 
 Item embeddings are unobserved, but chooser embeddings are used as covariates. 
 We train models on a confounded portion of the data and measure prediction quality on a held-out confounded subset as well as a counterfactual portion with uniformly random choice sets. 
For IPW, we estimate choice set propensities via per-item logistic regression, multiplying item propensities to get set propensities.

To measure prediction quality, we use the mean relative position of the true choice in the list of predictions sorted in descending probability order. 
A value of 1 says that the true choices were all predicted as most likely. As confounding strength increases, prediction quality increases in the confounded data for logit, MNL, and CDM, while decreasing on counterfactual data (\Cref{fig:synthetic-cdm}). For logit and MNL, IPW leads to models that generalize better to counterfactual data. For CDM, IPW correctly prevents the illusion of increased performance with more confounding (although variance caused by IPW appears to result in a small dip in performance at low confounding). Since preference ignorability is satisfied, IPW is unnecessary for MCDM: regression with the correctly specified model successfully generalizes despite confounding.

\begin{figure}[t]
\centering
\includegraphics[width=\columnwidth]{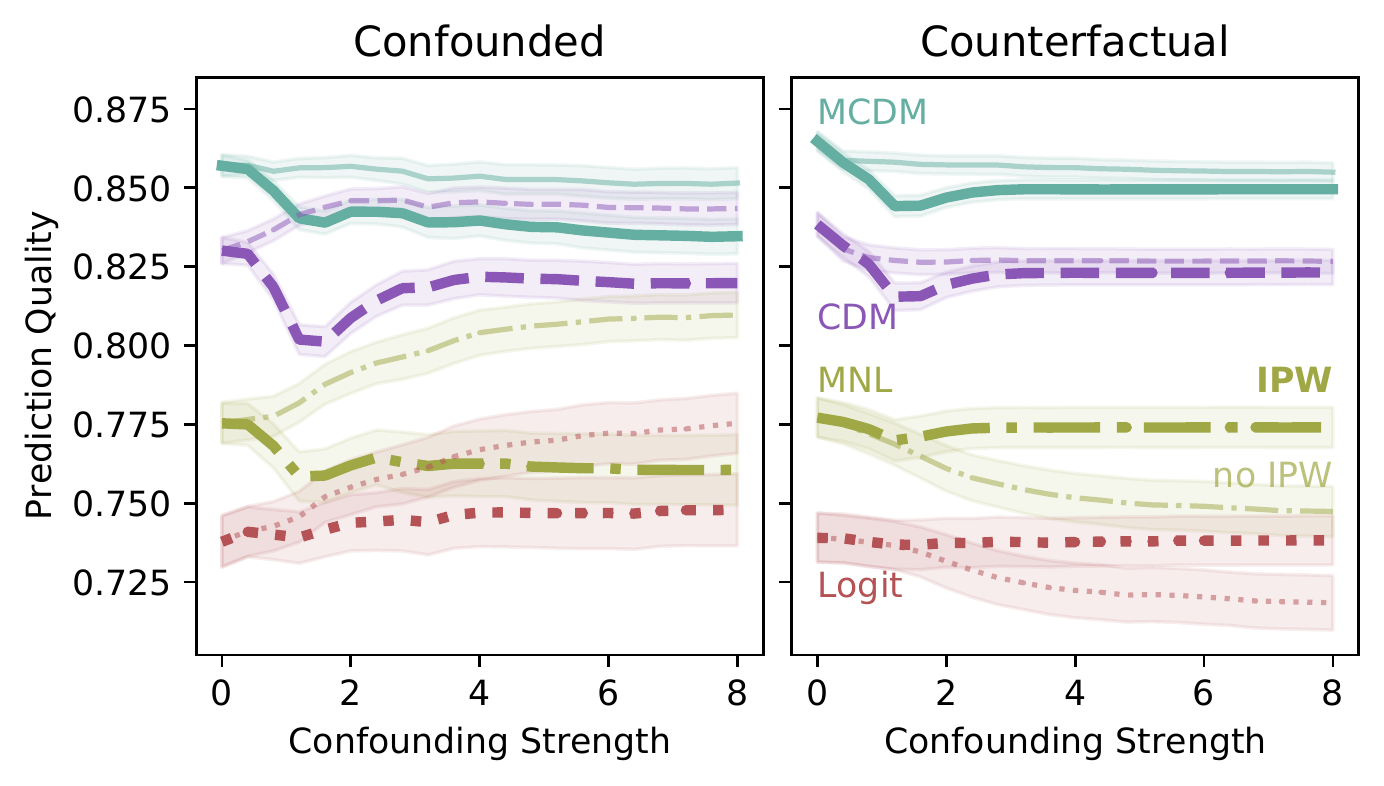}
\caption{Mean prediction quality of models on synthetic data with both context effects and choice set confounding, with IPW (bold) and without IPW (light). Left: out-of-sample predictions on data with confounding. Right: counterfactual predictions of models trained on confounded data. Shaded regions show standard error over 16 trials.}\label{fig:synthetic-cdm}
\end{figure}

\xhdr{Empirical data with chooser covariates}\label{sec:expedia}
We now consider the \textsc{expedia} hotel choice dataset~\cite{kaggle2013expedia} from \Cref{sec:wild}, using five hotel features: star rating, review score, location score, price, and promotion status. This allows us to use feature-based choice models (CL, CML, LCL, and MLCL; \Cref{tab:models,tab:context_models}). The dataset includes information about chooser searches, such as the number of adults and children in their party, which likely have strong effects on choice sets (i.e., search results). This is an excellent testbed for IPW since these covariates are likely informative about choice sets, making choice set ignorability more reasonable than preference ignorability. 

We do not have counterfactual choices for the \textsc{expedia} data, but we still consider several types of analysis. 
First, we recall the results from \Cref{tab:lrts} to see if apparent context effects are accounted for by chooser covariates. 
There, in contrast to the \textsc{SF} datasets, context effects still appear significant after controlling for covariates. 
In fact, context effects provide a larger likelihood boost than the chooser covariates. 
Thus, either 
(1) there are true context effects or
(2) the chooser covariates in \textsc{expedia} do not satisfy preference ignorability
(or both).
Based on the nature of the covariates, (2) seems very likely: the number of children in the chooser's party and the length of their stay are unlikely to fully describe hotel preferences.

\begin{figure}[t]
\centering
\includegraphics[width=\columnwidth]{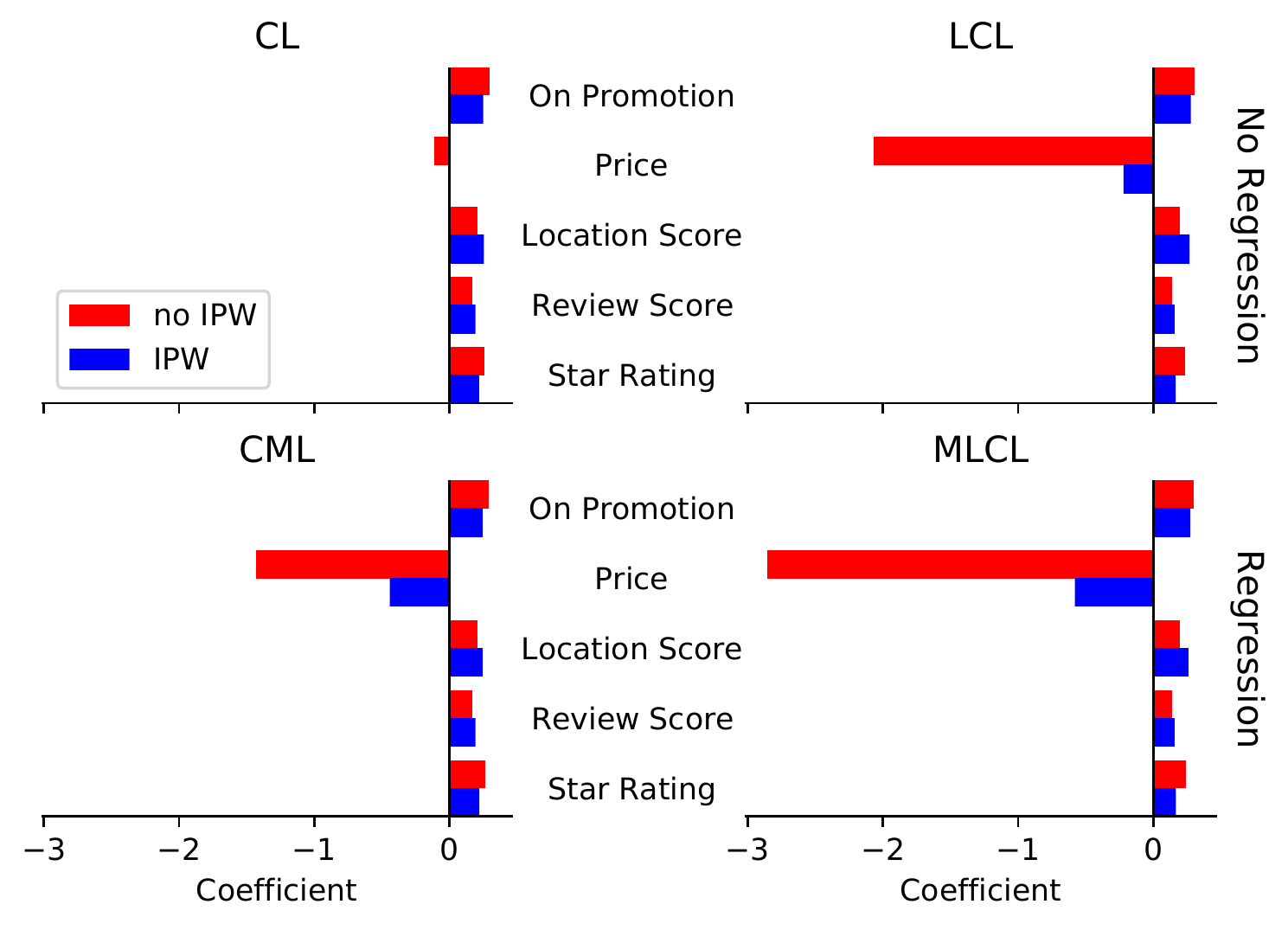}
\caption{Preference coefficients $\bm \theta$ in \textsc{expedia} for CL and LCL (top row, no regression); and CML and MLCL (bottom row, with regression), with and without IPW. A higher coefficient means choosers prefer higher values of the feature.}\label{fig:expedia-pref-coeffs}
\end{figure}

\begin{figure}[t]
\centering
\includegraphics[width=0.9\columnwidth]{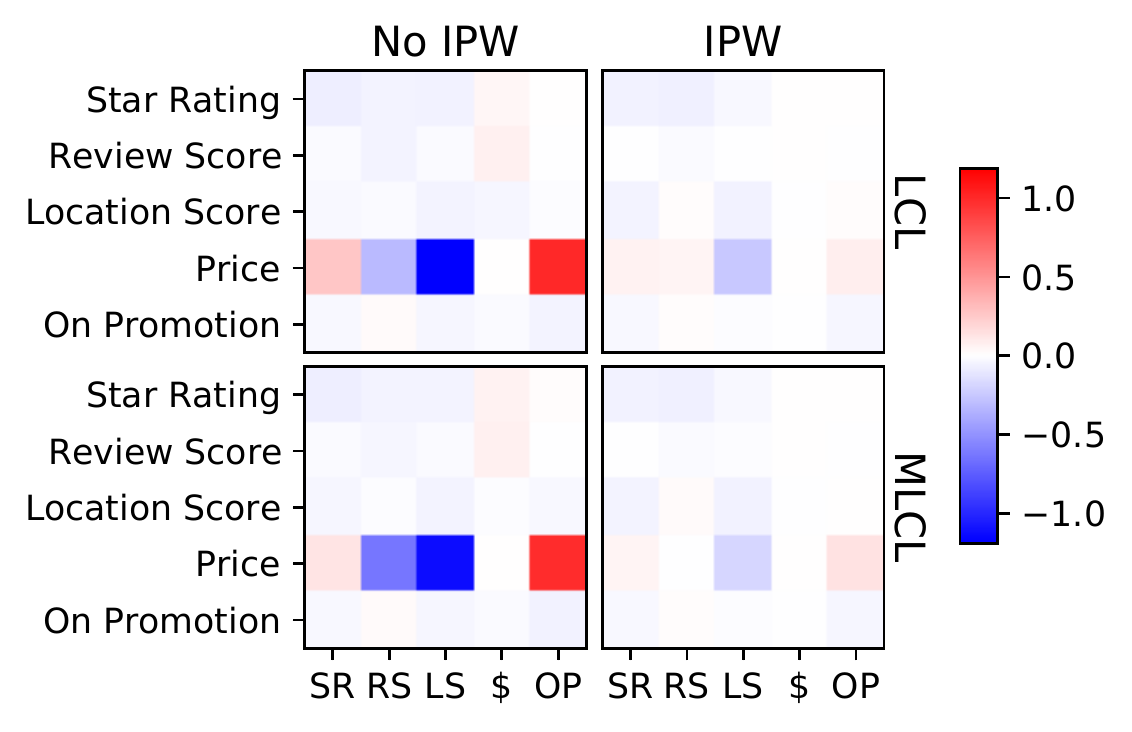}
\caption{LCL and MLCL context effect matrix $A$ in \textsc{expedia} with and without IPW. A higher value means choosers prefer a row feature more in a set where the mean column feature (abbreviated) is high; 0 indicates no context effect.}\label{fig:expedia-context-matrix}
\end{figure}

Since regression is inconclusive, we also apply IPW. 
To learn choice set propensities, we use a probabilistic model of the mean feature vectors of choice sets. 
We assume these vectors follow a multivariate Gaussian conditioned on chooser covariates, with mean $W \bm{x_a} + \bm{z}$ for some $W \in \R^{d_y\times d_x}, \bm{z} \in \R^{d_y}$. 
Given observed mean choice set vectors $\bm{y_C}$ and corresponding chooser covariates $\bm{x_a}$, we compute the maximum-likelihood $W$, $\bm{z}$, and covariance matrix (see \Cref{app:affine-gaussian}). 
This model gives us propensities for any $(a, C)$ pair. 

Using IPW with these propensities dramatically decreases the negative impact of high price in all four models (\Cref{fig:expedia-pref-coeffs}).
After adjusting for confounding, the models indicate that users are more willing to book more expensive hotels.   
This makes sense if Expedia is recommending relevant hotels: among a set of hotels matching a user's desired characteristics (such as location and star rating), we would expect them to select the cheapest option. On the other hand, if we presented users with a set of random hotels, location and star rating might play a stronger role in determining their choice, since a random set might have many cheap hotels that are undesirable for other reasons. In addition to the preference coefficients, IPW affects the context effect matrix $A$ in the LCL and MLCL (\Cref{fig:expedia-context-matrix}). In both models, IPW decreases (but does not entirely eliminate) the strong price context effects. This is evidence that some of the apparent context effects in the dataset are due to choice set confounding. 

Finally, the estimated likelihoods of the models under IPW are significantly better than without IPW (\Cref{tab:expedia-ipw}). We normalize the IPW-weighted log-likelihood by the sum of the IPW weights, which provides an estimate of what the IPW-trained model's log-likelihood would be given random sets. The gap between likelihood with no IPW and estimated likelihood with IPW dwarfs the gaps between different choice models, indicating that accounting for choice set confounding makes the data much more consistent with the random utility maximization principle underlying all four models. (By \Cref{thm:confounding_strength}, choice set confounding can result in choice systems far from rational behavior, even when choosers are rational.)

\begin{table}
  \caption{Log-likelihoods and estimated random-set log-likelihoods with IPW on \textsc{expedia}.
  After adjusting for confounding, the data is far easier to explain. 
  }\label{tab:expedia-ipw}
\begin{tabular}{lrr}
\toprule
\bfseries{Model} & \bfseries{Confounded} & \bfseries{IPW-adjusted}\\
\midrule
CL & $-839499$ & $-786653$\\
CML & $-838281$ & $-785753$\\
LCL & $-837154$ & $-784770$\\
MLCL & $-835986$ & $-783928$\\
\bottomrule
\end{tabular}
 \end{table}


\section{Managing without covariates}\label{sec:other_methods}

So far, we have used chooser covariates to correct for choice set confounding. However, in some choice data, there are no covariates available, or we are not willing to make ignorability assumptions. Here, we show what can be done in this setting.

\subsection{Within-distribution prediction}
Unfortunately, by \Cref{thm:confounding_strength}, it is impossible to determine whether IIA violations are caused by choice set confounding or true context effects in the absence of chooser information. 
Nonetheless, we can still exploit IIA violations---whatever their origin---to improve prediction, as long as we are careful not to make counterfactual predictions. This is essentially what researchers developing context effect models~\cite{ragain2016pairwise,seshadri2019discovering,bower2020salient,tomlinson2020learning,rosenfeld2020predicting} have been doing 
(without a framework for understanding the possibility of choice set confounding and the associated risks for counterfactual prediction). 
Beyond emphasizing a need for caution, we also establish a duality between models accounting for context effects and models accounting for choice set confounding; specifically, we show that a model equivalent to the CDM---which was designed with context effects in mind---can be derived purely from the perspective of choice set confounding. 

In a multinomial logit (MNL), we learn a latent parameter vector $\bm{\gamma_i}$ for each item $i \in \U$ and model utilies as $u_i(a) = \bm{x_a}^T\bm{\gamma_i}$ (omitting the intercept term). Suppose we don't have any chooser covariates, 
but we know choice set assignment depends on choosers.
We could then use the choice set itself as a surrogate for user covariates (e.g., one covariate could be ``someone who is offered item $i$''). 
Let $\bm{1_{C_a}}$ be a binary encoding of the choice set ${C_a}$ of a chooser $a$ (a length $|\U|$ vector with a 1 in position $i$ if $i \in C_a$).
Consider treating $\bm{1_{C_a}}$ as a substitute for the user covariates $\bm{x_a}$.
Then the MNL model is
\[
\textstyle \Pr(i \mid C_a) = \frac{\exp(\bm{1_{C_a}}^T \bm{\gamma_i})}{\sum_{j \in C_a} \exp(\bm{1_{C_a}}^T \bm{\gamma_j})}.
\]
The utility of $i$ in set $C_a$ is $\sum_{j \in C_a} \gamma_{ij}$, which is exactly the CDM (with self-pulls, since the sum is over $C_a$ rather than $C_a\setminus i$), a model designed to capture choice-set-dependent utilities.  
Thus, the CDM can either be thought of as accounting for pairwise interactions between items or using the choice set as a stand-in for user covariates. 

One natural question this duality raises is how the set of choice systems expressible by CDM (or other context-effect models) compares to the choice systems induced by mixed populations of IIA choosers with choice set confounding, which take the form
\begin{equation}
\textstyle \Pr(i \mid C) = \sum_{a\in \A} \Pr(a \mid C) \frac{\exp(u_i(a))}{\sum_{j \in C} \exp(u_j(a))}.\label{eq:confounded_choice_system}
\end{equation}
Mixtures of logits such as \cref{eq:confounded_choice_system} are notoriously hard to analyze (even in two-component case~\cite{chierichetti2018learning}), so no simple equivalence between a context-effect model and such a mixture is likely. In fact, \cref{eq:confounded_choice_system} is even trickier than standard mixed logit (\cref{eq:mixed_logit}), since the mixture weights depend on the choice set.

Nonetheless, some progress in this direction is possible. 
Here, we provide an instance where the LCL approximates a choice system induced by choice set confounding (of the form of \cref{eq:confounded_choice_system}). Recall that the LCL has utilities $u_i(a, C) = (\bm{\theta} + A\bm{y_C})^T\bm{y_i}$, where $\bm{y_C}$ is the mean feature vector over the choice set. If we make Gaussian assumptions on the distribution of features and on choice set assignment, and if chooser utilities are inner products of chooser and item vectors, then the LCL is a mean-field approximation to the induced choice system. In particular, we assume choice sets are generated to be similar to items the chooser $a$ would like (as in a recommender system) by sampling items from a Gaussian with mean $\bm{x_a}$.
\begin{theorem}\label{thm:lcl_mean_field}
	Let items and choosers both be represented by vectors in $\R^d$. Suppose chooser covariates $\bm{x_a}$ are distributed in the population according to a multivariate Gaussian $\mathcal{N}(\bm{\mu}, \Sigma_0)$, and a choice set for chooser $a$ is constructed by sampling $k$ items from the multivariate Gaussian $\mathcal{N}(\bm{x_a}, \Sigma)$. Additionally, assume choosers have the utility function $u_i(a, C) = \bm{x_a}^T \bm{y_i}$.  Then the expected chooser given a choice set $C$, $\bm{x_a^*} = \E[\bm{x_a} \mid C]$, has LCL choice probabilities, with
\begin{align*}
\textstyle  \bm{\theta} = \frac{1}{k} \Sigma(\Sigma_{0}+\frac{1}{k} \Sigma)^{-1} \bm{\mu},\quad  A = \Sigma_{0}(\Sigma_{0}+\frac{1}{k} \Sigma)^{-1}.
\end{align*}
\end{theorem}
Thus, the LCL can either be thought of as a context effect model, or as an approximation to the choice system induced by recommender-style preferred item overrepresentation. 

\subsection{Counterfactuals for known choosers}\label{sec:known_choosers}
To make counterfactual predictions without chooser covariates or insufficiently descriptive covariates (preventing us from applying IPW or regression), we develop a clustering method for the challenge of choice set confounding. 
Suppose a recommender system suggests two sets of movies to two users: $\{\textsf{Romance A}, \textsf{Romance B}\}$ to $a_1$ and $\{\textsf{Drama A}, \textsf{Drama B}\}$ to $a_2$. While we know nothing about $a_1$ or $a_2$, we might be inclined to think $a_1$ is likely to pick \textsf{Romance A} from $\{\textsf{Romance A}, \textsf{Drama A}\}$, while $a_2$ is likely to pick $\textsf{Drama A}$ from the same choice set.
Similar to the CDM derivation in the previous section, the choice set is a signal for chooser preferences. We can also apply collaborative filtering principles, with the distinction that instead of thinking that similar users like similar items, we assume similar choosers are shown similar choice sets.
There is a limitation, though, as this approach only lets us make predictions for choosers who appear in the original dataset. 
While there are many ways of using information from choice set assignment, we highlight an approach for the case where we have corresponding types of choosers and items (e.g., ``romance fans'' for ``romance movies'').

Suppose that choosers are more likely to have an item in their choice set if it matches their type. 
Define the $m \times n$ matrix $R$, where $R_{ij} = 1$ if the $i$th choice set includes item $j$ and $R_{ij} = 0$ otherwise.
We can think of $R$ as the upper right block of the adjacency matrix of a bipartite graph between choosers and items, in which an edge $(a, i)$ means that $i$ is in $a$'s choice set. With fixed choice set inclusion probabilities for each type, clustering choosers into types based on their choice sets is then an instance of the bipartite stochastic block model (SBM) recovery problem~\cite{larremore2014efficiently,abbe2017community}. 

In \Cref{thm:sbm}, we apply a classic exact recovery result due to \citeauthor{mcsherry2001spectral}~\cite{mcsherry2001spectral} to show how a choice system with discrete types can be deconfounded without access to chooser covariates (i.e., knowledge of type membership), but any bipartite SBM clustering algorithm could be used (see~\citeauthor{abbe2017community}~\cite{abbe2017community} for a survey of SBM results).

\begin{theorem}\label{thm:sbm}
Suppose items and choosers are jointly split into $k$ types. Let $s$ be the smallest number of items or choosers of any type and let $n = |A| + |\U|$. Suppose that for each chooser $a\in A$, $i\in \U$ is included in $a$'s choice set with probability $p$ if $a$ and $i$ are of the same type and with probability $q$ otherwise.

There exists a constant $C$ such that for large enough $n$, if
\begin{equation}
\textstyle s(p-q)^2 > C k  \left(\nicefrac{n}{s} + \log\nicefrac{n}{\delta} \right)\label{eq:mcsherry_cond},
\end{equation}
then w.p.\ $1-\delta$, we can efficiently learn the type of every item and every chooser given a dataset $\D$ with one choice from each $a \in A$.
\end{theorem}

While McSherry's algorithm has strong theoretical guarantees, 
a more practical implementation is spectral co-clustering~\cite{dhillon2001co}, 
which performs well for our purposes.
Once we recover type memberships, we train separate models for each type of chooser and use 
the model for a chooser's type for deconfounded counterfactual predictions.


\subsection{Empirical data without chooser covariates}\label{sec:yoochoose}
We apply our spectral method to the \textsc{yoochooose} online shopping dataset~\cite{ben2015recsys}. The dataset consists of all items clicked on in a session and an indicator of whether each item was purchased. We consider each purchase to be a choice from the set of all items viewed in the session. We group items by category (e.g., sports equipment) removing those with fewer than 100 purchases, leaving 29 categories. 

We then perform spectral co-clustering~\cite{dhillon2001co} on the choice set matrix $R$ with 2 to 10 chooser clusters and train a separate logit on each cluster. We ignore the item clusters. 
We compare against random clustering with the cluster sizes found by spectral clustering and mixed logit with the same number of components.

Spectral clustered logit describes the data much better than random clustering or even mixed logit (\Cref{fig:yoochoose-clustering}). 
Note that the clusters are based only on choice set assignment, not choice behavior. 
In contrast, mixed logit bases its mixture components solely based on choices. 
The strong performance of spectral clustering indicates that choice sets are informative about preferences, and our use of this information 
is much easier than learning a mixture model.

\begin{figure}[t]
\includegraphics[width=0.9\columnwidth]{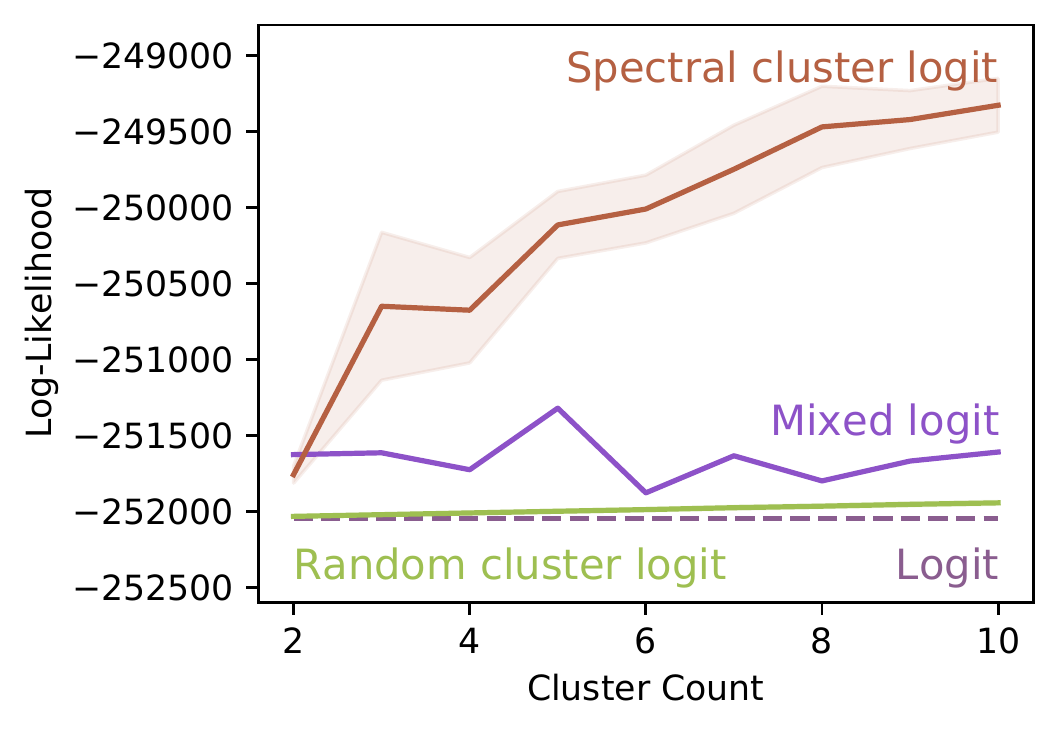}
\caption{\textsc{yoochoose} log-likelihood comparison. Spectral and random cluster results are averaged over eight trials, with one standard deviation shaded.}\label{fig:yoochoose-clustering}
\end{figure}


\section{Discussion}

Choice set confounding is widespread
and can affect choice probability estimates, alter or introduce context effects, and lead to poor generalization. 
Existing models ignoring chooser covariates are particularly susceptible, but plugging in covariates is not a universal solution. 
We saw that covariates may be more informative about choice sets than preferences, making IPW more viable than regression. 
An important contribution is formalizing and demonstrating choice set confounding,
 as it has significant implications for discrete choice modeling. 
For instance, initial research on the SF transportation data used extensive nested logit modeling to account for IIA violations~\cite{koppelman2006self}, which we can manage with choice set confounding.

Our methods are a first step in addressing confounding.
A challenge was learning choice set propensities for IPW. 
Simple logistic regression can work for binary treatments, but estimating exponentially many choice set propensities is difficult.
In \textsc{expedia}, we learned a distribution over mean choice set feature vectors as an approximation.
Other methods for learning set assignment probabilities would be valuable. 
Instrumental variables are another causal inference approach~\cite{hernan2006instruments} that could be used in our setting,
but identifying instruments for choice data is difficult.
Alternatively, a matching approach~\cite{imbens2004nonparametric} could compare pairs of similar choosers with different choice sets.
Other directions for future investigation include rigorous methods of detecting choice set confounding, or verifying that it has been successfully accounted for, and of testing assumptions.

\begin{acks} 
This research was supported by ARO MURI, ARO Awards W911NF19-1-0057 and 73348-NS-YIP, NSF Award DMS-1830274, the Koret Foundation, and JP Morgan Chase \& Co. We thank Spencer Peters for helpful discussions.
\end{acks}

\bibliographystyle{ACM-Reference-Format}
\bibliography{main}

\clearpage
\appendix
\section{Proofs}\label{app:proofs}
\begin{proof}[Proof of \Cref{obs:unbiased}]
	Conditioning over choosers yields $\Pr(i \mid C) = \sum_a \Pr(i \mid a, C) \Pr(a \mid C)$. Meanwhile, $\E_a[\Pr(i \mid a, C)] = \sum_{a} \Pr(i \mid a, C) \Pr(a)$. These are equal if condition (1) holds (since independence also implies $\Pr(a) = \Pr(a \mid C)$). If condition (2) holds, then we directly have $\E_a[\Pr(i \mid a, C)] = \Pr(i \mid C)$. See \Cref{ex:pets} for an instance where this equality fails when neither (1) nor (2) hold.
\end{proof}

\begin{proof}[Proof of \Cref{thm:confounding_strength}]
First, notice that any universal logit choice probabilities aggregated over a population can be expressed by set-dependent utilities $u_i^*(C)$ for each $C \subseteq \U, i \in C$.
For every choice set $C \subseteq \U$, construct a chooser $a_C$ with fixed utilities $u_i(a_C) = u_i^*(C)$. Let $\Pr(C \mid a_C) = 1$ and $\Pr(C' \mid a_C) = 0$ for all other $C' \ne C$. The choice probabilities of this mixture with chooser-dependent sets is the same as in the original system, and the mixture has finitely many ($2^{|\U|} - 1$) components, one for each nonempty choice set $C$.
\end{proof} 

\begin{proof}[Proof of \Cref{thm:ipw}]
	We need (1) in order for IPW (and therefore $\tilde D$) to be well-defined. Fix $i$ and $C$. Consider the coefficient of $\log \Pr_\theta(i \mid C)$ in $\ell(\theta; \D^*)$. In expectation, this term appears $|\D|\Pr(i, C)$ times. Expanding this:
 \begin{align*}
    |\D|\Pr(i, C) &= |\D|\sum_{a\in A} \Pr(i, C \mid a) \Pr(a)\\
    &= |\D|\sum_{a\in A} \Pr(i\mid C, a)\Pr(C \mid a) \Pr(a)\\
    &= \frac{|\D|}{|\C_\D|}\sum_{a\in A} \Pr(i\mid C, a) \Pr(a),
  \end{align*} 
  where the last step follows from $\D^*$ having uniformly random choice sets.
  Now consider the coefficient of $\log \Pr_\theta(i \mid C)$ in $\ell(\theta; \tilde \D)$. By IPW, this coefficient is
  \begin{align*}
     \sum_{\substack{(a, C', i) \in \D\\i'=i, C'=C}}\frac{1}{|\C_\D|\Pr(C\mid \bm{x_a})} &= \frac{1}{|\C_\D|}\sum_{a \in A}\sum_{\substack{(a', C', i') \in \D\\a' = a, C'=C, i'=i}}\frac{1}{\Pr(C\mid \bm{x_a})}.
   \end{align*} 
   In expectation, the sample $(a, C, i)$ occurs $|\D|\Pr(a, C, i)$ times. Additionally, by choice set ignorability, $\Pr(C \mid \bm{x_a}) = \Pr(C \mid a)$. We thus have that the expected coefficient is
   \begin{align*}
     &\frac{1}{|\C_\D|}\sum_{a \in A}\frac{1}{\Pr(C\mid \bm{x_a})}|\D|\Pr(a, C, i) \\
     &= \frac{|\D|}{|\C_\D|}\sum_{a \in A}\frac{1}{\Pr(C\mid a)}\Pr(a)\Pr(C \mid a)\Pr(i \mid C, a)\\
     &= \frac{|\D|}{|\C_\D|}\sum_{a \in A}\Pr(a)\Pr(i \mid C, a),
   \end{align*}
   which matches the coefficient in $\ell(\theta; \D^*)$. Since the expected coefficients agree for all $i$ and $C$, we then have the equality.
\end{proof}

\begin{proof}[Proof of \Cref{thm:regression}]
  By the consistency of the MLE, as $|\D|\rightarrow \infty$, parameter estimates for a correctly specified choice model converge to the true parameters. Thus, estimated choice probabilities also converge:
  \begin{align*}
     \lim_{|\D| \rightarrow \infty}\hat \Pr(i \mid \bm{x_a}, C) &= \Pr(i \mid \bm{x_a}, C)\\
     &= \Pr(i \mid a, \bm{x_a}, C)\tag{by preference ignorability}\\
     &= \Pr(i \mid a, C).
   \end{align*} 
\end{proof}

\begin{proof}[Proof of \Cref{thm:lcl_mean_field}]
	Observing the choice set gives us a noisy measurement of $\bm{x_a}$, which we can adjust using our knowledge of the distribution of $\bm{x_a}$. The posterior of a Gaussian with a Gaussian prior is also Gaussian---in particular, $\bm{x_a} \mid C$ is Gaussian, with mean
\[\E[\bm{x_a} \mid C] = \Sigma_{0}\left(\Sigma_{0}+\frac{1}{k} \Sigma\right)^{-1} \bm{y_C}+\frac{1}{k} \Sigma\left(\Sigma_{0}+\frac{1}{k} \Sigma\right)^{-1} \bm{\mu}\]
\cite[Section 3.4.3]{duda2001pattern}. Thus, the expected chooser $\bm{x_a^*}$ has utilities
\begin{equation*}
	u_i(a^*, C) = \left[\Sigma_{0}\left(\Sigma_{0}+\frac{1}{k} \Sigma\right)^{-1} \bm{y_C}+\frac{1}{k} \Sigma\left(\Sigma_{0}+\frac{1}{k} \Sigma\right)^{-1} \bm{\mu}\right]^T\bm{y_i}.
\end{equation*}
This is exactly an LCL with $\bm{\theta}$ and $A$ as claimed.
\end{proof}

\begin{proof}[Proof of \Cref{thm:sbm}]
	Consider the bipartite graph whose left nodes are choosers and whose right nodes are items, each split into blocks according to their type. The choice set assignment process above defines a bipartite SBM on this graph with intra-type probabilities $p$ and inter-type probabilities $q$ (between chooser nodes and item nodes). Recovering types from choice sets can then be viewed as an instance of the planted partition problem~\cite{mcsherry2001spectral}. 

	We can thus directly\footnote{Notice that $s(p-q)^2$ is a lower bound on the squared 2-norm of the columns of the SBM edge probability matrix required by \cite[Theorem 4]{mcsherry2001spectral}. Additionally, we use the crude variance upper bound $\sigma^2=1$ for simplicity.} apply Theorem 4 of \citeauthor{mcsherry2001spectral}~\cite{mcsherry2001spectral} to achieve the desired result given \Cref{eq:mcsherry_cond}, with the caveat that algorithm is random and succeeds with probability $1/k$. 

	Repeating the algorithm $ck$ times achieves failure probability $(1-\frac{1}{k})^{ck} \le 1/e^c$, which is smaller than $\delta$ if $c > \log (1 / \delta)$. We can thus make $\delta$ smaller by a factor of 2 (absorbing this into the constant $C$ in \cref{eq:mcsherry_cond}) and we are left with the guarantee as stated, only increasing the running time by a factor $k\log (1 / \delta)$.
\end{proof}

\section{Affine-mean Gaussian Choice Set Model}\label{app:affine-gaussian}
For estimating choice set propensities in \textsc{expedia}, we model the distribution of mean choice set features using an affine-mean Gaussian. Here, we show how this model can be easily estimated from data. 

\begin{proposition}

Given a dataset $\D$, the model $\bm{y_C} \sim \N(W \bm{x_a} + \bm{z}, \Sigma)$ is identifiable iff there are $m+1$ choosers in $\D$ with affinely independent covariates. If the model is identified, the maximum likelihood parameters $W^*, \bm{z}^*$ are the solution to the least-squares problem
\begin{equation}
  (W^*, \bm{z}^*) = \argmin_{\substack{W \in \R^{n \times m}\\\bm{z}\in \R^n}} \sum_{(a, C) \in \D} \|\bm{y_C} - (W\bm{x_a}+\bm{z})\|_2^2,
\end{equation}
which have the closed form:
\begin{align}
  W^* &= \left[\sum_{(a, C) \in \D} (\bm{y_C}-\bm{y_\D})\bm{x_a}^T\right] \left[\sum_{(a, C) \in \D} (\bm{x_a}-\bm{x_\D})\bm{x_a}^T\right]^{-1}\\
  \bm{z}^* &= \bm{y_\D} - W^* \bm{x_\D},
\end{align}
where $\bm{x_\D} = \nicefrac{1}{|\D|} \sum_{(a, C) \in \D} \bm{x_a}$ and $\bm{y_\D} = \nicefrac{1}{|\D|} \sum_{(a, C) \in \D} \bm{y_C}$.

Additionally, the maximum likelihood covariance matrix is the sample covariance:
\begin{equation}
  \Sigma^* = \frac{1}{|\D|} \sum_{(a, C) \in \D}  (\bm{y_C} - W^*\bm{x_a}-\bm{z}^*)(\bm{y_C} - W^*\bm{x_a}-\bm{z}^*)^T.
\end{equation}
\end{proposition}
\begin{proof}[Proof sketch]
  This can be derived following the same steps as the standard Gaussian MLE proof (with a bit of extra matrix calculus): (1) take partial derivatives of the log-likelihood with respect to $W$ and $\bm{z}$, (2) set them to zero, (3) solve for $\bm{z}$, (4) plug this in to solve for $W$, (5) do the same to solve for $\Sigma$ in its partial derivative. This works since the log-likelihood is still convex after adding in the affine trasformation. We omit the details as they are tedious and unenlightening.
\end{proof}

\section{Experiment Details}\label{sec:app_experiment}
We implemented all choice models with PyTorch and (except mixed logit) train them using Rprop
with no minibatching to optimize the log-likelihood for 500 epochs or until convergence (squared gradient norm $< 10^{-8}$), whichever comes first. 
We use $\ell_2$ regularization with coefficient $\lambda = 10^{-4}$ for all models to ensure identifiability. 
For mixed logit, we use an expectation-maximization (EM) algorithm~\cite{train2009discrete} with a one hour timeout. Our code, results, and links to data are available at

\centerline{\url{https://github.com/tomlinsonk/choice-set-confounding}.} 

\begin{table}[h]
  \caption{Regularity violations in \textsc{sf-work} and \textsc{sf-shop}, impossible under mixed logit. Including additional item(s) appears to increase the probability that \textsf{DA} or \textsc{DA/SR} is chosen. The differences are significant according to Fisher's exact test (\textsc{sf-work}: $p=6.5 \times 10^{-9}$, \textsc{sf-shop}: $p=0.005$).}  \label{tab:sf-reg}
\centering
  \begin{tabular}{lrr}
  \toprule
  \multicolumn{3}{l}{\textsc{sf-work}}\\
  \bfseries{Choice set ($C$)} & $\Pr(\textsf{DA} \mid C)$ & $N$\\
  \midrule
  \{\textsf{DA}, \textsf{SR 2}, \textsf{SR 3+}, \textsf{Transit}\} & 0.72 & 1661\\
  \{\textsf{DA}, \textsf{SR 2}, \textsf{SR 3+}, \textsf{Transit}, \textsf{Bike}\} & 0.83 & 829\\
  \midrule[0.08em]
  \multicolumn{3}{l}{\textsc{sf-shop}}\\
  \bfseries{Choice set ($C$)} & $\Pr(\textsf{DA/SR} \mid C)$ & $N$\\
  \midrule
  \begin{tabular}{l}
    \{\textsf{DA}, \textsf{DA/SR}, \textsf{SR 2}, \textsf{SR 3+},\\
     \quad \textsf{SR 2/SR 3+}, \textsf{Transit}\}
   \end{tabular} & 0.17 & 534\\
  \begin{tabular}{l}
    \{\textsf{DA}, \textsf{DA/SR}, \textsf{SR 2}, \textsf{SR 3+}, \\
    \quad  \textsf{SR 2/SR 3+}, \textsf{Transit}, \textsf{Bike}, \textsf{Walk}\}
   \end{tabular} & 0.23 & 1315\\
  \bottomrule
   \multicolumn{3}{l}{\footnotesize{\textsf{DA}: drive alone. \textsf{SR}: shared ride, number indicates car occupancy.}}\\[-0.4em]
  \multicolumn{3}{l}{\footnotesize{Slashes indicate different mode used for outbound and inbound trips.}}
  \end{tabular}

  \begin{tabular}{lrr}
  \toprule
  
  \end{tabular}

  
\end{table}

\end{document}